\documentclass[accepted]{uai2022} % after acceptance, for a revised
\usepackage[american]{babel}
\usepackage{natbib} % has a nice set of citation styles and commands
\usepackage{mathtools} % amsmath with fixes and additions
\usepackage{booktabs} % commands to create good-looking tables
\usepackage{tikz} % nice language for creating drawings and diagrams

\usepackage{xcolor}

\usepackage{amsmath}
\usepackage{amssymb}
\usepackage{amsthm}
\usepackage{graphicx}
\usepackage{algorithm,algorithmicx,algpseudocode}
\usepackage{subcaption}
\usepackage{bbm}

\newtheorem{theorem}{Theorem}
\newtheorem{corollary}{Corollary}
\newtheorem{lemma}{Lemma}

\DeclareMathOperator{\N}{\texttt{N}}
\DeclareMathOperator{\adapt}{adapt}
\DeclareMathOperator{\approved}{approved}
\DeclareMathOperator{\prespec}{pres}
\DeclareMathOperator{\Var}{Var}

\title{Sequential algorithmic modification with test data reuse}

\author[1]{Jean Feng}
\author[2]{Gene Pennello}
\author[2]{Nicholas Petrick}
\author[2]{Berkman Sahiner}
\author[3]{Romain Pirracchio}
\author[2]{Alexej Gossmann}
\affil[1]{%
    Department of Epidemiology and Biostatistics\\
    University of California, San Francisco
}
\affil[2]{%
    U.S. Food and Drug Administration
}
\affil[3]{%
    Department of Anesthesiology\\
    University of California, San Francisco
  }

\begin{document}
\maketitle

\begin{abstract}
\vspace{-0.2cm}
After initial release of a machine learning algorithm, the model can be fine-tuned by retraining on subsequently gathered data, adding newly discovered features, or more.
Each modification introduces a risk of deteriorating performance and must be validated on a test dataset.
It may not always be practical to assemble a new dataset for testing each modification, especially when most modifications are minor or are implemented in rapid succession.
Recent works have shown how one can repeatedly test modifications on the same dataset and protect against overfitting by (i) discretizing test results along a grid and (ii) applying a Bonferroni correction to adjust for the total number of modifications considered by an adaptive developer.
However, the standard Bonferroni correction is overly conservative when most modifications are beneficial and/or highly correlated.
This work investigates more powerful approaches using alpha-recycling and sequentially-rejective graphical procedures (SRGPs).
We introduce novel extensions that account for correlation between adaptively chosen algorithmic modifications.
In empirical analyses, the SRGPs control the error rate of approving unacceptable modifications and approve a substantially higher number of beneficial modifications than previous approaches.
\end{abstract}

\vspace{-0.3cm}
\section{Introduction}
\label{sec:intro}
\vspace{-0.2cm}

Before a machine learning (ML) algorithm is approved for deployment, its performance is usually evaluated on an independent test dataset.
If the ML algorithm is modified over time, its performance may change.
There are no guarantees on \textit{how} the performance may evolve when the model developer is allowed to introduce modifications in an unconstrained manner.
For instance, algorithmic modifications that reduce computational costs may negatively impact model accuracy or precision, and improvements along an aggregate performance metric may come at the cost of worse performance for certain minority subgroups and exacerbate issues of algorithmic fairness.
To check that a proposed modification is acceptable for deployment, the current approach is to run a hypothesis test on a new test dataset, separate from the original one \citep{Feng2020-ev}.
The null hypothesis is that the modification is not acceptable; a modification is approved if we successfully reject the null.
Nevertheless, large high-quality test datasets are often hard to acquire, particularly in the medical setting.
A major motivation for this work comes from the FDA's recent interest in letting medical device developers update ML-based software, while still ensuring its safety and effectiveness \citep{US_Food_and_Drug_Administration2019-kt}.

% In such settings, we can check that a proposed modification is acceptable for deployment by running a hypothesis test on independently collected collected test data \citep{Feng2020-ev}.
% The null hypothesis is that the modification is not acceptable; a modification is approved if we successfully reject the null.
% Nevertheless, large high-quality test datasets are often hard to acquire, particularly in the medical setting.
When labeled data are expensive and/or difficult to collect, it is tempting to reuse an existing test dataset for determining the acceptability of an algorithmic modification.
The danger of test data reuse is that the model developer can learn aspects of the test data when it is used in a sequential and adaptive manner, creating dependencies between algorithmic modifications and the holdout data.
For instance, the model developer may inadvertently incorporate spurious correlations in the test data to attain over-optimistic performance estimates.
This feedback loop introduces bias to the performance evaluation procedure, and adaptively defined hypothesis tests can have drastically inflated Type I error rates \citep{Gelman2017-kh, Thompson2020-ut}.

Recent works protect against inappropriate test data reuse and overfitting by reducing the amount of information released by the testing procedure \citep{Russo2016-ms}.
The two main approaches are to either coarsen the test outputs along a grid of values \citep{Blum2015-hv, Rogers2019-br} or to perturb the test results with random noise using differential privacy techniques \citep{Dwork2015-da, Feldman2018-eo}.
However, existing methods require immensely large datasets to provide protection against overfitting with theoretical guarantees \citep{Rogers2019-br}.
Our aim is to design valid test data reuse procedures for smaller sample sizes that still have sufficiently high power to approve good algorithmic modifications.
Our focus is on methods that coarsen the test results.
In fact, we consider the extreme case of coarsening where the procedure releases a single bit of information, e.g. whether or not the modification was approved.

When test results are coarsened, the adaptive modification strategy can be described as a tree.
As such, one can view the test data reuse problem as a multiple hypothesis testing problem: If we control the family-wise error rate across the entire tree, we control the probability of approving one or more unacceptable modifications.
Existing procedures perform a Bonferroni correction with respect to the size of this tree \citep{Blum2015-hv, Rogers2019-br}.
Nevertheless, the Bonferroni correction is known to be conservative.
Instead, we can gain significant power using alpha-recycling \citep{Burman2009-pp} and accounting for correlation between test statistics \citep{Westfall1993-ac}.
Indeed, we expect algorithmic modifications to be highly correlated when there is significant overlap between their training data and similarities in their training procedures.

In this paper, we design valid test data reuse procedures based on sequentially rejective graphical procedures (SRGPs) \citep{Bretz2009-bt, Bretz2011-hd, Bretz2011-mu}.
Although SRGPs are a well-established technique for testing many \textit{pre-specified} hypotheses, many of these procedures cannot be applied when the hypotheses are \textit{adaptively defined in sequence}.
The main challenge is that many nodes in the tree of hypotheses are not observed.
As such, we introduce two novel SRGPs that are able to account for correlation between adaptively-defined algorithmic modifications without needing to observe these ``counterfactual'' hypotheses.
The first procedure accounts for correlation between observed nodes in the tree using a fixed-sequence testing procedure.
The second procedure is based on the fact that analysts are not adversarial in practice, i.e. they will not purposefully use prior results to overfit to the test data \citep{Mania2019-pt, Zrnic2019-lt}.
We leverage this fact by requiring the model developer to pre-specify a hypothetical online learning procedure.
The SRGP then utilizes the similarity between the adaptive and pre-specified modifications to improve testing power.
Both methods can be applied to test black-box ML algorithms and, importantly, are suitable for smaller sample sizes, e.g. those commonly found in medical settings.
In empirical analyses, our procedures protect against overfitting to the test data and approve a higher proportion of acceptable modifications than existing approaches.
Code will be publicly available on Github.

\vspace{-0.3cm}
\section{Problem Setup}
\label{sec:mtp}
\vspace{-0.1cm}
Suppose the test dataset is composed of $n$ independently and identically distributed (IID) observations $(X_1, Y_1) (X_2, Y_2), \ldots, (X_n, Y_n) \in \mathcal{X}\times\mathcal{Y}$ drawn from the target population.
Consider a model developer who adaptively proposes a sequence of $T$ algorithmic modifications $\{\hat{f}_1^{\adapt},...,\hat{f}_T^{\adapt}\}$, where each modification is a model that predicts some value in $\mathcal{Y}$ given input $X$.
Given criteria for defining the acceptability of a modification, our goal is to approve as many acceptable modifications as possible while controlling the probability of approving an unacceptable modification.
Because the decision to approve a modification can be framed as a hypothesis test, a procedure for approving adaptively-defined modifications is equivalent to testing a sequence of adaptively-defined hypotheses $H_{1}^{\adapt}, ..., H_{T}^{\adapt}$.
Moreover, control of the online family-wise error rate (FWER) in the strong sense, i.e. for any configuration of the null hypotheses, implies control over of the rate of approving at least one unacceptable modification.

There are various ways to define acceptability and their corresponding hypothesis test.
For example, we may define a modification $\hat{f}$ to be acceptable as long as its expected loss is smaller than that of the original model $\hat{f}_0$.
So given a real-valued loss function $\ell$, we would test the hypothesis
\begin{align*}
	H_{j}^{\adapt}: \mathbb{E}\left(\ell\left(\hat{f}_j^{\adapt}(X), Y \right)\right) \ge  \mathbb{E}\left(\ell\left(\hat{f}_0(X), Y \right)\right)
\end{align*}
at each iteration $j = 1,...,T$.
If we require monotonic improvement in the model performance, we can test if the $j$-th modification is superior to the most recently approved modification $\hat{f}_{j}^{\approved}(X)$, i.e.
\begin{align*}
	H_{0}^{\adapt}:
	\mathbb{E}\left(\ell\left(\hat{f}_{j}^{\adapt}(X), Y \right)\right) \ge  \mathbb{E}\left(\ell\left(\hat{f}_{j}^{\approved}(X), Y \right)\right).
\end{align*}
Finally, one may also consider multidimensional characterizations of model performance (e.g. model performance within subgroups) and define acceptability as a combination of superiority and non-inferiority tests \citep{Feng2020-ev}.
The testing procedures described below only depend on the p-values, so we leave the specific definition of acceptability unspecified until the experimental section.

To limit the adaptivity of the model developer, we consider procedures that sequentially release a single bit of information for each test: One means that the modification is approved and zero means it is not.
Consequently, modifications proposed by any adaptive strategy can be described as a bifurcating tree, where $\hat{f}_{a_t}$ is the \textit{nonadaptive} model tested at time $t$ for the history of approvals $a_t \in \{0,1\}^{t-1}$, $H_{a_t}$ is the associated \textit{nonadaptive} hypothesis test, and $p_{a_t}$ is its marginal p-value.
While one can regard this set of $(2^T - 1)$ hypotheses tests as prespecified, we are only able to observe a specific path along this tree.
The unobserved hypotheses are counterfactuals.
As such, we need a multiple testing procedure (MTP) that controls the FWER for any adaptively chosen path along the tree \textit{without knowing the exact nature of the counterfactual hypothesis tests}.

%, not all MTPs can be applied in our setting because we only observe a specific path along this adaptive tree.
%The unobserved hypotheses are counterfactuals.
%We need an MTP that does not require knowing the exact nature of the counterfactual hypothesis tests.

A simple approach is to perform a uniform Bonferroni correction for the size of the entire tree.
However, the standard Bonferroni procedure has low power because it ignores correlations between models and allocates substantial test mass to hypotheses that are unlikely to be considered.
Next we describe procedures that can achieve much higher power.

\vspace{-0.3cm}
\subsection{Sequentially rejective graphical procedures (SRGPs)}
\label{sec:SRGPs}
\vspace{-0.1cm}

We can design more powerful test data reuse procedures by building on \textit{sequentially rejective graphical procedures} (SRGPs), which use directed graphs to define a wide variety of iterative MTPs such as gatekeeping procedures, fixed sequence tests, and fallback procedures \citep{Bretz2009-bt}.
SRGPs traditionally assume the set of hypotheses $\{H_j: j \in I\}$ is prespecified and known.
The graph initially contains one node for each elementary hypothesis $H_j$, where each node is associated with a non-negative weight $w_{j}(I)$.
The initial node weights, which are constrained to sum to one, control how the total alpha is divided across the elementary hypotheses and correspond to a set of \textit{adjusted} significance thresholds $c_{j}(I)$.
We reject elementary hypothesis $H_j$ in the current graph if its marginal p-value $p_{j}$ is smaller than $c_{j}(I)$.
For instance, a standard Bonferroni correction is represented by the initial weights of $w_{j}(I) = 1/|I|$ for all $j \in I$ and significance thresholds $c_{j}(I) = w_{j}(I) \alpha$.
In addition, the graph contains directed edges where the edge $H_j$ to $H_k$ is associated with weight $g_{j, k}(I)$ for $j, k \in I$ and edge weights starting from the same node must sum to one.
When an elementary hypothesis $H_j$ is rejected, its node is removed from the graph and its weight is propagated to its children nodes.
This redistribution of test mass, also known as alpha-recycling, increases the power for testing the remaining hypotheses and strictly improves upon simpler procedures that do not use recycling.
More specifically, the weight of the edge from $H_j$ to $H_k$, denoted $g_{j, k}(I)$, represents how much of $H_j$'s node weight will be redistributed to $H_k$ if $H_j$ is rejected.
So when $H_j$ is removed, the new weight for hypothesis $H_k$ for $k \in I' = I \setminus \{j\}$ is
\begin{align}
	w_{k}(I') = w_{k}(I) + g_{j, k}(I) w_{j}(I).
	\label{eq:weight_prop}
\end{align}
Outgoing edges for all remaining nodes are also renormalized to sum back to one.
The SRGP continues until no more hypotheses can be rejected.
See Figure~\ref{fig:srgp_example} for an example.
% To simplify analyses in our adaptive setting, we will assume that there are no edge weights going from a later to an earlier hypothesis.
% As such, the weights for a node should only depend on the rejection status of earlier hypotheses.
% (This isn't a strong requirement, but it'll make things easier to reason about.)

\begin{figure}
	\centering
	\vspace{-0.3cm}
	\includegraphics[width=0.8\linewidth]{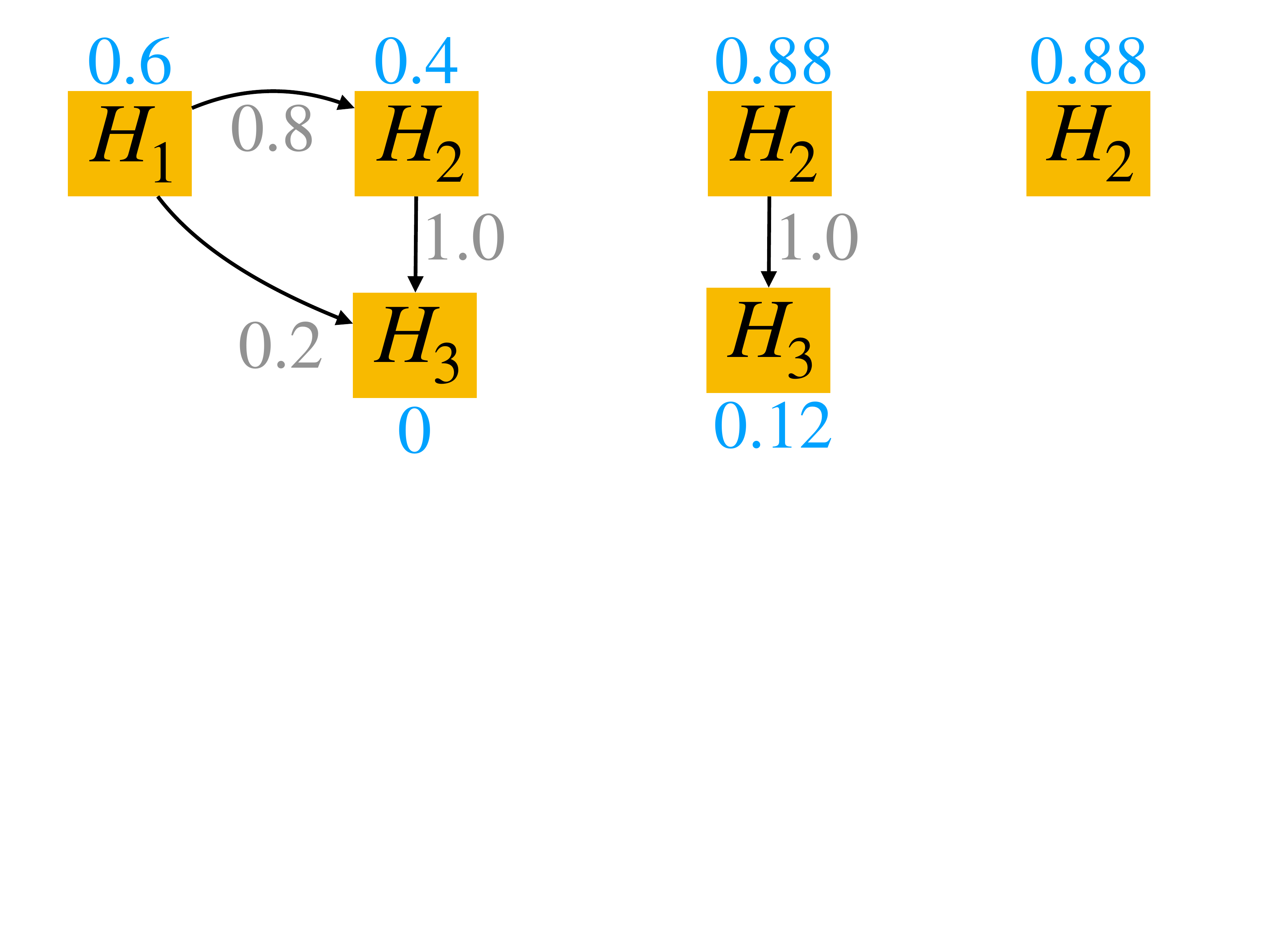}
	\vspace{-2.8cm}
	\caption{Example sequentially-rejective graphical procedure (SRGP)  for hypotheses $\{H_1, H_2, H_3\}$ with initial graph on the left, the middle graph after $H_1$ is rejected, and the right graph after $H_3$ is rejected. Node and edge weights are blue and gray, respectively.}
	\vspace{-0.5cm}
	\label{fig:srgp_example}
\end{figure}

Assuming all test reports are binary, we can describe adaptive test data reuse as the following \textit{prespecified} SRGP.
Let $I_t$ be the set of hypotheses remaining at time $t$, where $I_0$ is the initial set.
%For notational convenience, we also include a ``start'' node $\Ss_{a_0}$ in $I_0$.
%The start node is assigned a non-negative weight but is immediately removed at time $t = 1$ to propagate its weight.
% Note however that this node cannot be removed from the graph.
We only consider SRGPs with nonzero edge weights from hypothesis $H_{a_t}$ to the sequence of hypotheses that would be tested upon rejection of $H_{a_t}$ but prior to the next rejection (though one may also consider more complex recycling procedures).
For such SRGPs, the graph of hypotheses has the tree structure seen in Figure~\ref{fig:tree}, where the only edges in the tree are between hypotheses $H_{a_t}$ and $H_{a_{t'}}$ for $t < t'$ and
\vspace{-0.1cm}
\begin{align}
	a_{t', j} = a_{t, j} \mathbbm{1}\{j < t\} + \mathbbm{1}\{j = t\}.
	\label{eq:edges}
\end{align}
Note that the SRGP tree is \textit{not} the same as the bifurcating tree for generating hypotheses, as the former describes how alpha is recycled.

The model developer must prespecify all initial node weights $w_{a_t}(I_{0})$.
%This is necessary to ensure that the node weights sum to one.
A simple approach is to perform a uniform Bonferroni correction across all nodes in the graph.
We can achieve more power by assigning larger weights to nodes that are more likely to be tested.
For example, if the model developer knows that all their modifications will be approved, they should set the initial node weight for $H_{()}$ and all edge weights along the top path in Figure~\ref{fig:tree} to one.

\begin{figure*}
	\centering
	\vspace{-0.3cm}
	\begin{subfigure}{0.3\textwidth}
		\centering
		\includegraphics[width=0.8\linewidth]{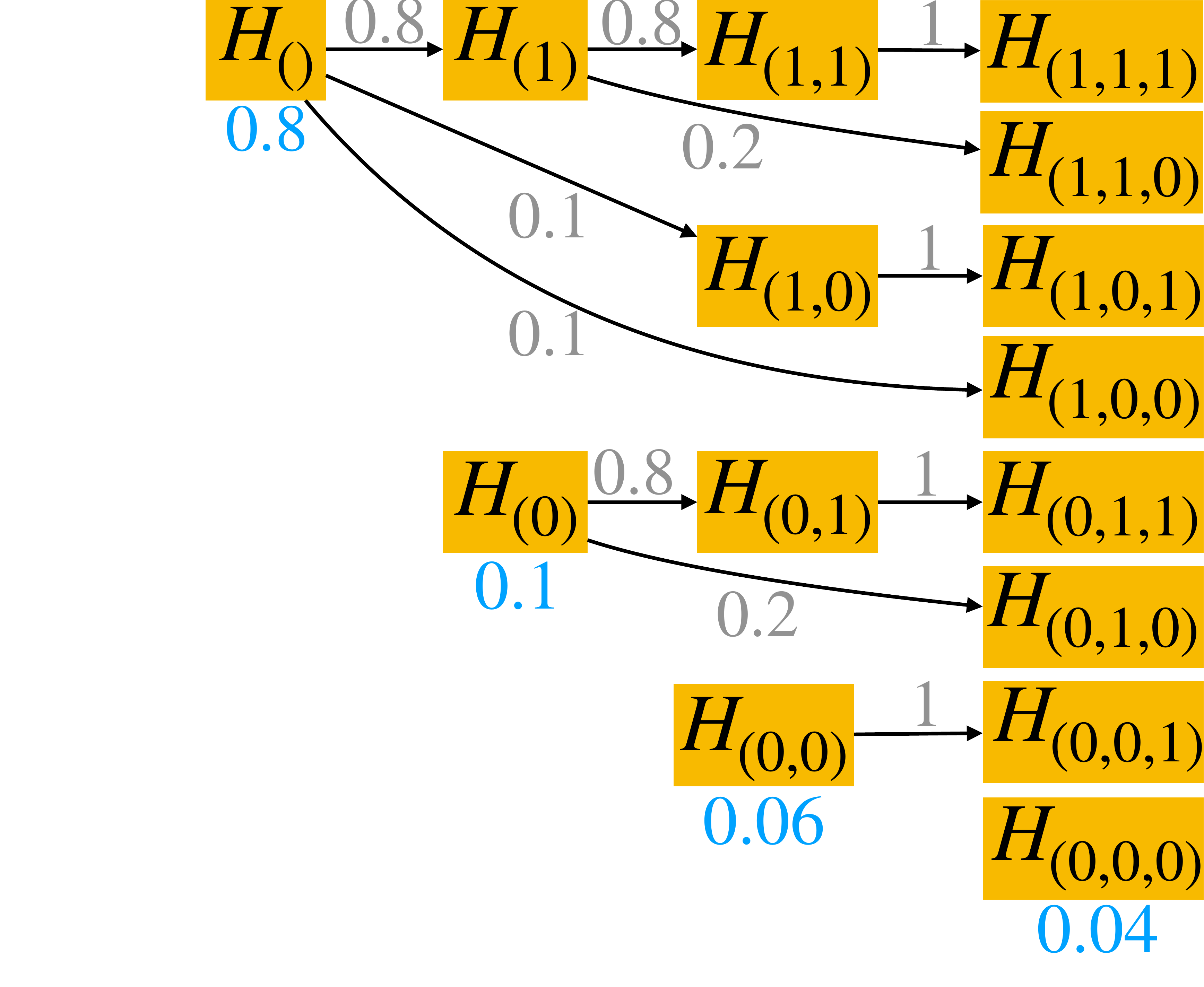}
		\caption{}
		\label{fig:tree}
	\end{subfigure}
	\begin{subfigure}{0.3\textwidth}
		\centering
		\vspace{0.1cm}
		\hspace{-0.5cm}
		\includegraphics[width=0.8\linewidth]{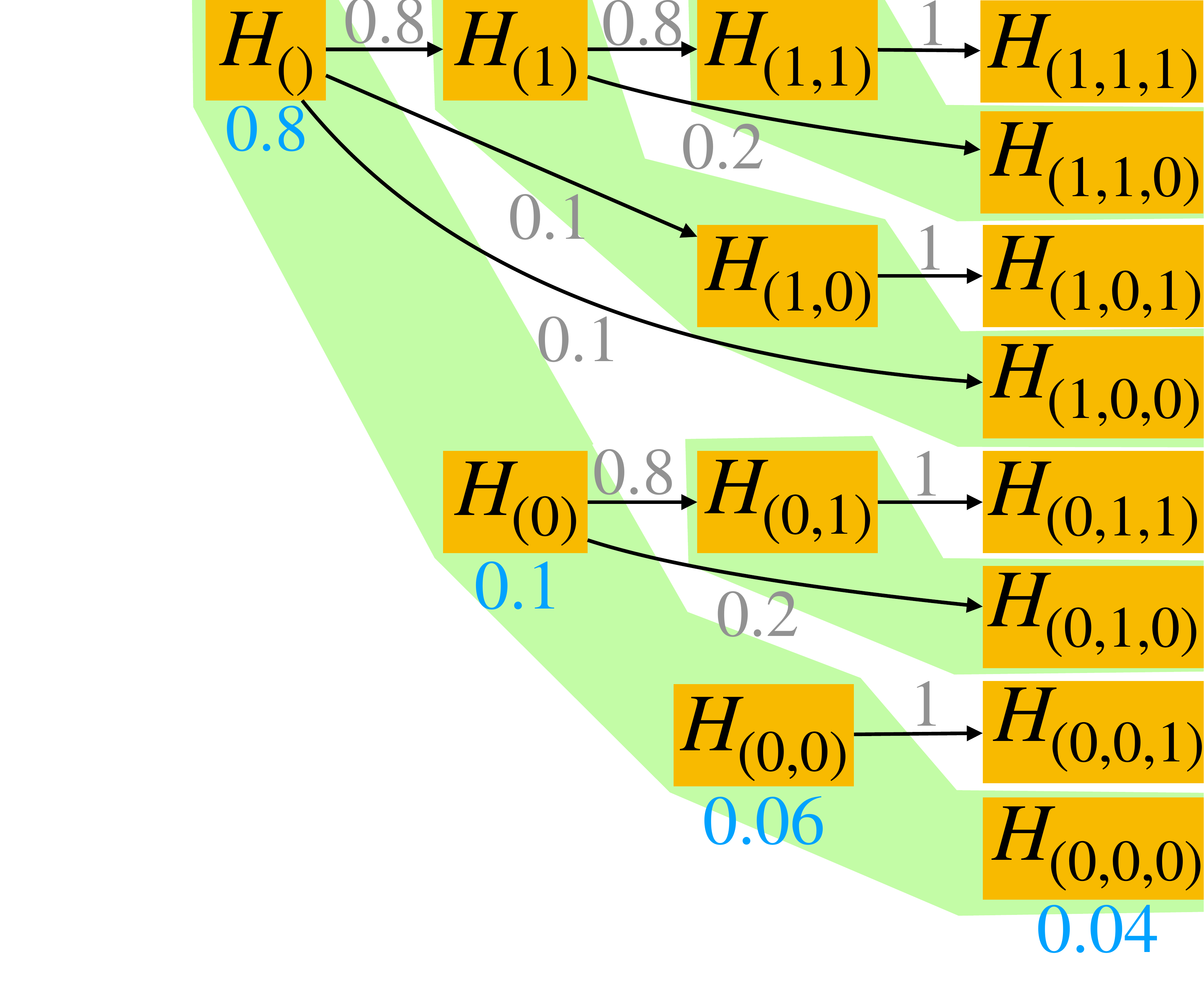}
		\vspace{-0.2cm}
		\caption{}
		\label{fig:tree_grouped}
	\end{subfigure}
	\begin{subfigure}{0.3\textwidth}
		\centering
		\hspace{-0.5cm}
		\includegraphics[width=0.8\linewidth]{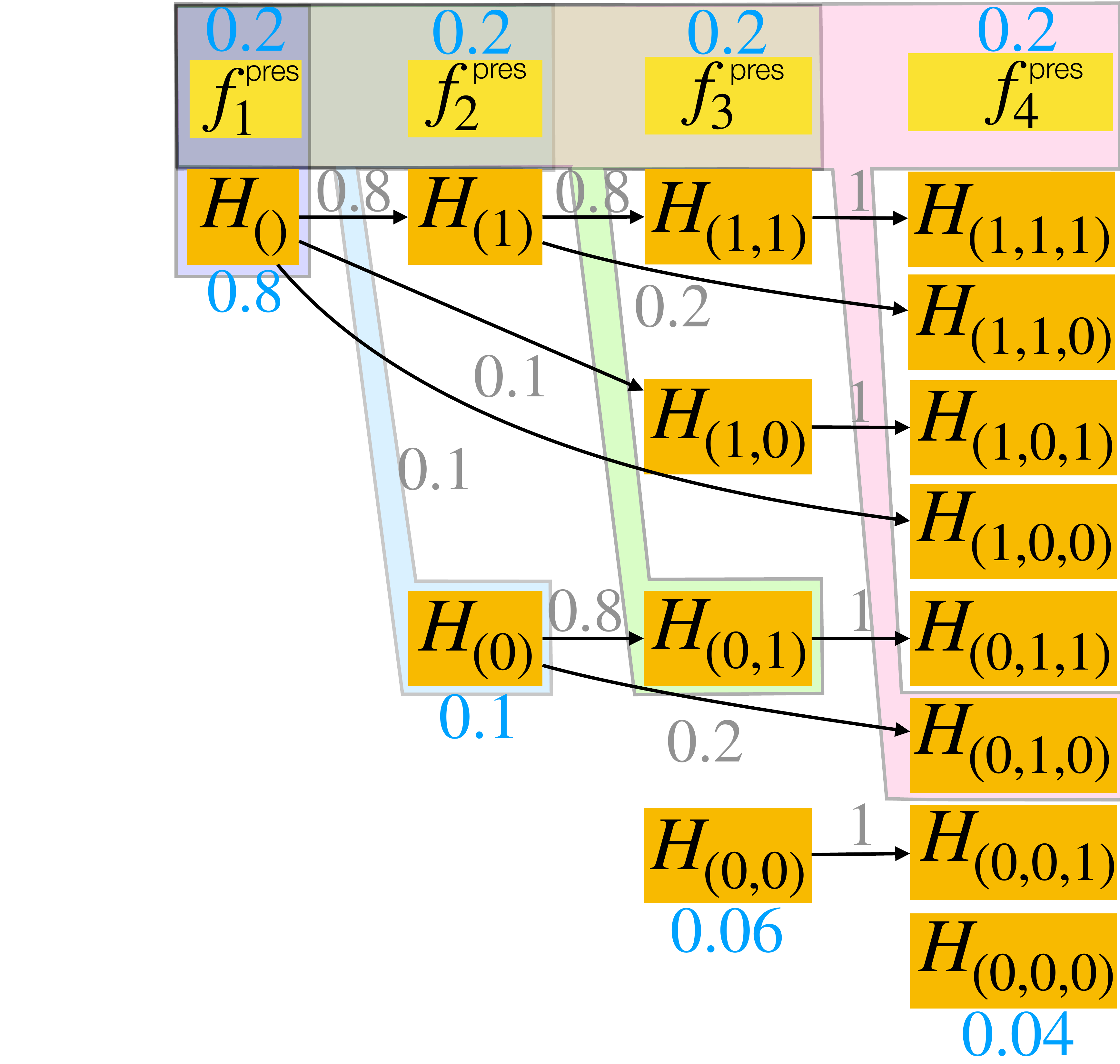}
		\vspace{-0.1cm}
		\caption{}
		\label{fig:tree_parallel}
	\end{subfigure}
	\vspace{-0.3cm}
	\caption{
		SRGPs for testing $T = 4$ adaptively-defined algorithmic modifications, where $H_{a_t}$ is the hypothesis for testing the adaptive modification given history $a_t \in \{0,1\}^{t - 1}$.
		(a) A SRGP based on a weighted Bonferroni test.
		(b) A SRGP that performs a fixed sequence test within each shaded subgroup.
		(c) A SRGP that adjusts for the correlation between the adaptively proposed modifications and those from a prespecified hypothetical updating procedure, denoted $\{f^{\prespec}_t: t = 1,...,T\}$.
		As an example, we indicate the correlation adjustments made along the path $H_{()}, H_{(0)}, H_{(0,1)}, H_{(0,1,0)}$.
	}
\end{figure*}

\begin{algorithm}[]
	\begin{algorithmic}
		\Require{Initialize $I_0$ as the set of all nodes in the prespecified tree; initialize $a_1 = ()$ and $\tau_1 = 0$; choose node weights $w_{a_{t'}}(I_0)$ for all $t'=1,2,\ldots,T$ and $a_{t'}\in\{0,1\}^{t'-1}$; and set $w_{a_{0}}(I_0) = 0$.}
		\Ensure{$\sum_{t', a_{t'}} w_{a_{t'}}(I_0) = 1$.}
		\For{$t=1,2,...,T$}
		\State{Specify edge weight $g_{a_{\tau_t}, a_{t}}$ that satisfies outgoing edge weight constraints.}
		\State{\texttt{\# Weight propagation}}
		\State{$w_{a_{t}}(I_t) = w_{a_{t}}(I_{t - 1}) + g_{a_{\tau_t}, a_{t}} w_{a_{\tau_t}}(I_{\tau_t})$}
		%		\Comment{Update weight for the hypothesis to be tested.}
		\For{all $t', a_{t'}$ such that $H_{a_{t'}}  \in I_t$ and $a_{t'} \neq a_t$}
		\State{\texttt{\# Other weights remain unchanged}}
		\State{$w_{a_{t'}}(I_t) = w_{a_{t'}}(I_{t - 1})$}
		%		\Comment{Other weights remain unchanged.}
		\EndFor

		\State{Let $p_{a_t}$ be the marginal p-value from testing $H_{a_t}$.}
		\State{Compute significance threshold $c_{a_t}(I_t)$ using \texttt{compute\_sig\_threshold}$(a_t, \{w_{a_{t'}}(I_t): t', a_{t'}\})$}

		\If{$p_{a_t} \le c_{a_t}(I_{t})$}
		\State{Report that  $\hat{f}_{t}^{\adapt}$ has been approved.}
		\State{\texttt{\# Remove node}}
		\State{$I_{t + 1} = I_{t} \setminus {a_{t}}$}
		\State{$\tau_{t + 1} = t$}
		\State{$a_{t + 1} = (a_{t}, 1)$}
		\Else
		\State{Report that  $\hat{f}_{t}^{\adapt}$ has not been approved.}
		\State{$\tau_{t + 1} = \tau_t$}
		\State{$a_{t + 1} = (a_{t}, 0)$}
		\EndIf
		\EndFor
	\end{algorithmic}
	\caption{
		A sequentially rejective graphical procedure (SRGP) that only outputs binary test reports for $T$ adaptive hypotheses given function \texttt{compute\_sig\_threshold}.
	}
	\label{algo:graph_update}
\end{algorithm}

The model developer will only need to incrementally reveal the edge weights.
%For notational convenience, let $g_{a_{t'}, a_{t}}$ denote the edge weight between the node with index $a_{t'}$ to $a_t$ for $t' < t$.
% Recall that in our problem setup $a_{t}$ is the history of binary test reports up to time $t$, and $a_{t}$ is also used as the index of the ML model $\hat{f}_{a_t}$ to be tested next.
Let $\tau_{t}$ denote the time of the latest approval prior to time $t$.
% That is, $a_{\tau_{t}}$ is the index of the last approved ML model $\hat{f}_{a_{\tau_{t}}}$.
At time $t$, the developer must specify the edge weights $g_{a_{\tau_{t}}, a_t}$ such that the outgoing weights from $a_{\tau_t}$ sum to no more than one.
(Note that the edge weight can be treated as a constant because the only relevant edge weight at time $t$ is $g_{a_{\tau_t}, a_t}(I_t)$ and its value is equal to $g_{a_{\tau_t}, a_t}(I_{t'})$ for all $t' < t$.)
As such, this procedure for specifying node and edge weights corresponds to a fully prespecified SRGP where a subset of the edge weights are revealed sequentially.
To make sure that this SRGP can be executed in the adaptive setting, we must be able to calculate the adjusted significance thresholds for the adaptive hypotheses given the current set of node weights without observing the counterfactual hypotheses.

The entire SRGP algorithm for testing adaptive algorithmic modifications is outlined in Algorithm~\ref{algo:graph_update}.
It accepts some function \texttt{compute\_sig\_threshold} that outputs the significance threshold for the adaptively chosen hypothesis given node weights in the current tree.
To prove that an SRGP with function \texttt{compute\_sig\_threshold} controls the FWER, we must show that it is a closed test procedure that satisfies the consonance property.
Recall that a closed test procedure uses the following recipe to control the FWER at level $\alpha$: it rejects an elementary hypothesis $H_{j}$ if the intersection hypothesis $H_K = \cap_{k \in K} H_{k}$ for every subset $K \subseteq I$ containing the elementary hypothesis $H_j$ is rejected at level $\alpha$ \citep{Lehmann2005-us}.
Moreover, a closed test satisfies the consonance property if the following is true for all $J \subseteq I$: if intersection hypothesis $H_{J}$ is rejected locally (i.e. its p-value is no more than $\alpha$), there exists some $j \in J$ such that $H_K$ can be rejected locally for all $K \subseteq J$ with $j \in K$ \citep{Gabriel1969-mk}.
In particular, it follows that the corresponding elementary hypothesis $H_j$ can be rejected by the closed test procedure.
When consonance holds, we can perform the closed test using a sequentially rejective (or ``shortcut'') procedure that iteratively rejects the elementary hypotheses \textit{without needing to test every intersection hypothesis} \citep{Hommel2007-mw}.
When hypothesis tests are fully prespecified, consonance makes closed testing more computationally efficient/tractable.
The consonance property is even more important in the adaptive setting because we can reject the adaptive hypotheses without observing counterfactual or future hypotheses.
As such, the consonance property of an SRGP in the adaptive setting is not simply for computational efficiency, but is necessary for being able to compute anything.

Below, we will describe three SRGPs for testing an adaptive sequence of algorithmic modifications, presented in order of increasing complexity.
Each differ in how \texttt{compute\_sig\_threshold} is defined.
To prove that the procedures satisfy the consonance property, it is sufficient to show that the following monotonicity condition holds \citep{Bretz2009-bt}:
For every pair of subsets $K, J \subseteq I$ where $K\subseteq J$ and $j\in K$, we have
\begin{align}
	c_{j}(J) \le c_{j}(K).
	\label{eq:monotonicity}
\end{align}
All proofs are provided in the Appendix.

\vspace{-0.3cm}
\subsection{Bonferroni-based SRGPs}
\label{sec:bonf_SRGP}
\vspace{-0.1cm}

We begin with the simplest SRGP that performs closed testing with a weighted Bonferroni-Holm correction based on node weights, which was originally proposed in \citet{Bretz2009-bt} to test a set of fully pre-specified hypotheses.
Nevertheless, this procedure can also be applied in the adaptive setting because the significance thresholds do not depend on observing the counterfactual hypotheses.
In particular, this procedure tests the $t$-th adaptive hypothesis given history $a_t$ by comparing its marginal p-value to the corrected significance threshold
$ c_{a_{t}}(I_{t}) = w_{a_{t}}(I_{t}) \alpha$.
Because this closed test satisfies the monotonicity condition, Algorithm~\ref{algo:graph_update} with this significance threshold controls the FWER for the adaptive hypotheses at level $\alpha$.

As a simple example, consider an SRGP that initially assigns Bonferroni-corrected weights to every node and selects nonzero edge weights.
This is more powerful than performing a standard Bonferroni correction without any alpha-recycling because the significance thresholds are monotonically non-decreasing at each iteration.

\vspace{-0.4cm}
\subsection{SRGPs with fixed sequence tests for correlated modifications}
\label{sec:SRGP_fixed}

In practice, algorithmic modifications are likely to be highly correlated.
In this case, a Bonferroni-based SRGP is conservative.
We can design more powerful SRGPs by taking into account correlation between the p-values.
\citet{Bretz2011-hd} proposed a procedure that calculates an inflation factor $c(I)$ for intersection hypothesis $I$ such that the probability there exists an elementary hypothesis $H_{\N}$ with marginal p-value $p_{\N}$ less than $c(I) w_{\N}(I) \alpha$, under the null $I$, is no more than $\alpha$.
\citet{Millen2011-pw} proposed a similar procedure but for test statistics and critical values.
Unfortunately, both procedures require knowing the exact correlation structure between all the hypotheses and checking that the monotonicity property holds.
% , since certain configuration of the weights or correlation structures may violate the monotonicity property.
This is not feasible in the adaptive setting.
To resolve these issues, we propose a new SRGP that (1) partitions the hypothesis tree into sequences of \textit{observed} hypotheses and (2) uses a fixed-sequence test within each subgroup.

We group together hypotheses that would be tested along a streak of failures immediately following a successful approval (Figure~\ref{fig:tree_grouped}).
That is, we define a subgroup for history $a_t \in \{0,1\}^{t-1}$ with  $a_{t,t-1} = 1$ as the hypotheses with histories $a_{t'} = (a_t, \vec{0})$ for any length zero vector, i.e.
\begin{align*}
\hspace{-1cm}
	G_{a_t} = \left \{
	H_{{a}_{t'}}: {a}_{t',i} = a_{t,i} \mathbbm{1}\{i\le t-1\},
	\forall i = 1,...,t' - 1,
	\forall t' \ge t  \right \}.
\end{align*}
To test intersection hypothesis $I$, we test each subgroup $G_{a_t} \cap I$ at level
$
\left(\sum_{H_{a_{t'}} \in G_{a_t} \cap I} w_{a_{t'}}(I) \right) \alpha.
$
We reject $H_I$ at level $\alpha$ if any of the subgroup-specific tests are rejected.
We can show that this controls the Type I error at level $\alpha$ using a union bound.
To test a subgroup, we test its hypotheses in the order they are revealed and spend up to the allocated alpha weight.
To satisfy the monotonicity property, the significance threshold $c_{a_{j}}(I)$ for $a_j \in G_{a_t} \cap I$ is defined as the maximum threshold that spends no more than the allocated alpha up to time $j$ for all subsets of hypotheses, i.e.
\begin{align}
\hspace{-0.8cm}
	\begin{split}
		& c_{a_{j}}(I) =\sup  \tilde{c} \\
		& \text{s.t. } \Pr\left(p_{a_k} > c_{a_k}(I) \forall a_k \in K, p_{a_j} < \tilde{c} | H_{K \cup \{a_j\}} \right)\\
		& \le  \Bigg [
		\sum_{\substack{a_k \in ((G_{a_t} \cap I) \setminus K) \\ k \le j}} \hspace{-0.3cm} w_{a_{k}}(I) \Bigg ]
		\alpha \quad \forall K \subseteq \{a_k : a_k \in G_{a_t} \cap I, k < j\}.
	\end{split}
	\label{eq:sig_thres_corr}
\end{align}
This expression is complicated because it handles arbitrary correlation structures between the p-values.
It greatly simplifies in certain cases.
For example, if we are performing one-sided Z-tests and the pairwise correlations of the model losses  are non-negative, \eqref{eq:sig_thres_corr} is equivalent to defining $c _{a_{j}}(I)$ as the solution to
\begin{align*}
	\begin{split}
		& \Pr\left(p_{a_{k}} > c_{a_{k}}(I) \forall k=t,...,j-1, p_{a_{j}} < {c}_{a_{j}}(I) | \cap_{k=t}^j H_{a_{k}} \right)\\
		&= \ w_{a_{j}}(I) \alpha.
	\end{split}
\end{align*}
Using the fixed sequence tests from above, we sequentially calculate the significance thresholds and test the adaptive hypotheses.
When a hypothesis is rejected, we remove its node and propagate its \textit{local} weight to its children nodes per \eqref{eq:weight_prop}.
% In particular, some weight will be propagated to another partition and we start a new fixed sequence test.
We can prove the monotonicity condition holds to establish the following result:
% To our knowledge, this result of recycling the local alpha is new in the SRGP literature.
\begin{theorem}[]
	Algorithm~\ref{algo:graph_update} with significance thresholds chosen using \eqref{eq:sig_thres_corr} controls the FWER for adaptively defined hypotheses at level $\alpha$.
	\label{thrm:ffs}
\end{theorem}

\vspace{-0.4cm}
\subsection{SRGPs with prespecified hypothetical model updates}
\label{sec:SRGP_natural}

The SRGPs in the above sections protect against the worst case scenario where the model developer is adversarial.
In practice, the model developer may have a plan for how they will update their model over time (i.e. continually refit the model on accumulating data) and will only make small adjustments based on test results.
As such, we do not expect the adaptively chosen model at iteration $t$ to stray far from the initial plan.
In the most extreme case, we may find that the model developer is not adaptive at all and follows the prespecified procedure perfectly; instead of correcting for $(2^{T} - 1)$ hypotheses, we would expect that the correction factor to be $O(T)$ instead.

To leverage this similarity assumption, we propose a novel SRGP that requires the model developer to prespecify a procedure for generating hypothetical model updates.
This prespecified procedure describes the \textit{exact} steps for how modifications would be generated, e.g. the data stream used, the number of training observations, and hyperparameter selection.
%We utilize the similarity between modifications generated by this prespecified updating procedure and adaptively-defined modifications to improve power for approving the latter.
These hypothetical model updates are included as additional nodes in the hypothesis graph and assigned positive node weights.
Their sole purpose is to improve power for approving the adaptively-defined model updates.
These model updates are never formally tested nor approved for deployment.
We also do not release \textit{any} information about their test performance, because doing so would increase the amount of information leaked to the model developer and the branching factor of the adaptive tree.

At each iteration, this SRGP constructs a confidence region for the performance of the $t$-th prespecified model update $\hat{f}_t^{\prespec}$ by spending its allocated alpha, accounting for its correlation with all prespecified models up to iteration $t - 1$.
It then tests the $t$-th adaptive model by accounting for its correlation with the prespecified models up to iteration $t$.
As such, the power for testing the adaptive modifications increase as their correlation with the prespecified updates increases.

%Note that the hypothetical, prespecified model updates are neither formally tested nor deployed.
%Their nodes are never removed from the graph and we never release \textit{any} information about their test performance, because doing so would increase the amount of information leaked to the model developer and the branching factor of the adaptive tree.
%Instead, we spend alpha on the prespecified nodes by leveraging the relationship between p-values and confidence intervals (CIs).

More formally, the critical value and significance threshold at time $t$ are calculated as follows.
Let $P_0$ denote the target population and $P_n$ denote the empirical distribution of the test dataset.
Here we consider a univariate performance measure $\psi$, where $\psi\left(\hat{f}, P \right)$ is the performance of model $\hat{f}$ with respect to distribution $P$.
It is straightforward to extend this procedure to multivariate performance measures (see the Appendix for an example).
Denote the deviation between the estimated and true performance as
$$
\xi_{t,n}^{\prespec} = \psi\left(\hat{f}_t^{\prespec}, P_n \right) - \psi\left(\hat{f}_t^{\prespec}, P_0 \right).
$$
For intersection hypothesis $I$, define $\tilde{I}$ as union of $I$ and all prespecified nodes.
Define critical value $z_t^{\prespec}(I)$ as the largest $\tilde{z}$ such that
\begin{align}
	\Pr\left(
	\xi_{t',n}^{\prespec} > z_{t'}^{\prespec}(I) \  \forall t' < t,
	\quad \xi_{t,n}^{\prespec} \le \tilde{z}
	\right) \le w_{t}^{\prespec}\left (\tilde{I} \right) \alpha.
	\label{eq:pres_ci}
\end{align}
% (Prespecified weights don't actually change with the intersection hypotheses because we never reject them.)
The significance threshold $c_{a_t}(I)$ for testing $H_{a_t}$ is defined as the largest $\tilde{c}$ such that
\begin{align}
\hspace{-0.2cm}
	\Pr\left(
	\xi_{t',n}^{\prespec} > z_{t'}^{\prespec}({I}) \ \forall t' \le t,
	\quad p_{a_t} \le \tilde{c} \mid H_{a_t}
	\right) \le w_{a_t}\left (\tilde{I} \right) \alpha.
	\label{eq:adapt_err}
\end{align}
Crucially, these calculations do not depend on observing counterfactual or future hypotheses.
Using a union bound, we can show that the Type I error for falsely rejecting the intersection hypothesis $I$ using the critical values defined above is bounded by the sum of the right hand sides of \eqref{eq:pres_ci} and \eqref{eq:adapt_err} for all $(t, a_t)$ in $I$.
Because the total weight in the graph is always one, we achieve Type I error control at level $\alpha$.
Using this idea, we can show that this SRGP indeed controls the FWER:
\begin{theorem}
	Algorithm~\ref{algo:graph_update} using significance thresholds defined using equations~\eqref{eq:pres_ci} and \eqref{eq:adapt_err} control FWER at level $\alpha$ for adaptively selected hypotheses.
	\label{thrm:parallel}
\end{theorem}

%For instance, we may prespecify that model updates will be trained on a monthly basis on all observations from an independent data stream that month.

%The prespecified updates, denoted ${f}_t^{\prespec}$, are not formally tested and, thus, are never approved or deployed.

% Nevertheless, the prespecified model updates can drastically improve power when the correlation between the adaptive and prespecified model updates is high.
% Indeed, previous works have used this idea of model similarity to reduce the correction factor for valid test data reuse \citep{Mania2019-pt, Zrnic2019-lt}; however, a major drawback is that they make assumptions about the model developer that cannot be verified.
% Here we provide a practical procedure that takes advantage of this similarity assumption.
% Our approach is to ask the model developer to prespecify a procedure for generating model updates, denoted ${f}_t^{\prespec}$.
% The SRGP will use the correlation between adaptive model updates to these prespecified updates to reduce the correction factor.

% Rather than spending alpha to test a null hypothesis, we will spend alpha to construct confidence intervals of the expected loss for the prespecified updates.
%Note that the prespecified nodes are never removed from the graph nor are their weights recycled (all other nodes can still be removed).
% We will only need to demonstrate the monotonicity condition holds for all node subsets $I$ that contain all prespecified nodes.
% The significance thresholds for testing the hypothesis nodes will depend on the weights assigned for the node subset $\tilde{I}$.

\vspace{-0.3cm}
\section{Simulation studies}
\label{sec:sim}
\vspace{-0.1cm}

We now present two simulation studies of model developers who adaptively propose modifications to their initial ML algorithm.
The developers aim to improve the model's area under the receiver operating characteristic curve (AUC) and quantify the performance increase as accurately as possible.
%In addition, the model developers would also like to quantify the performance increase as accurately as possible (i.e. the AUC has improved by 0.1); a numeric quantity is useful for many reasons, such as when communicating to model consumers the impact/utility of the new model update.
Because our adaptive test data reuse procedures only release a single bit of information at each iteration, we must carefully design the hypothesis tests to obtain a numeric bound on the performance improvement.
In particular, we define the $j$-th adaptive hypothesis test as
\begin{align}
	H_{0,j}^{adapt}: \psi\left (\hat{f}^{adapt}_j; P_0 \right) \le \psi\left (\hat{f}_0; P_0 \right) + \delta^{adapt}_j
	\label{eq:null_delta}
\end{align}
where $\psi(f, P)$ denotes AUC of model $f$ for distribution $P$ and $\delta^{adapt}_j \ge 0$ is the improvement difference that we are trying to detect.
To ensure the model performance tends to improve with each approval, we set $\delta^{adapt}_{j + 1} = \delta^{adapt}_{j} + 0.01$ whenever the $j $-th null hypothesis is rejected.
Note that one could consider more complicated hypotheses, each with their pros and cons.
For example, one can check that the modifications are strictly improving \textit{and} test for an improvement difference; however, this can be overly stringent.

The purpose of the first simulation study is to investigate FWER control.
We do this by simulating a model developer who tries to overfit to the test data based on the information released at each iteration.
The purpose of the second simulation study is to investigate power.
Here the model developer generally proposes good algorithmic modifications by continually refitting the model given an IID data stream.

In both simulations, we generate $X \in \mathbb{R}^{100}$ using a multivariate Gaussian distribution.
$Y$ is generated using a logistic regression model where the coefficients of the first six variables are 0.75 and all other model parameter are zero.
The modifications are also logistic regression models.
We evaluate the two SRGPs proposed in this paper---SRGP with fixed sequence tests (\texttt{fsSRGP}) and SRGP with hypothetical prespecified model updates (\texttt{presSRGP})---against relevant baseline comparators, including the standard Bonferroni procedure (\texttt{Bonferroni}) and the Bonferroni-based SRGP (\texttt{bonfSRGP}).
The weights in the SRGPs were defined such that the first outgoing edge (a successful rejection of the hypothesis) is 0.8 and for each subsequent edge, it was assigned 0.8 of the remaining weight.
Unless specified otherwise, all the MTPs control the FWER at level $\alpha=0.1$.
Details for deriving test statistics and significance thresholds are provided in the Appendix.

\vspace{-0.3cm}
\subsection{Verifying FWER control}

\begin{figure}
\hspace{-0.4cm}\includegraphics[width=0.53\textwidth]{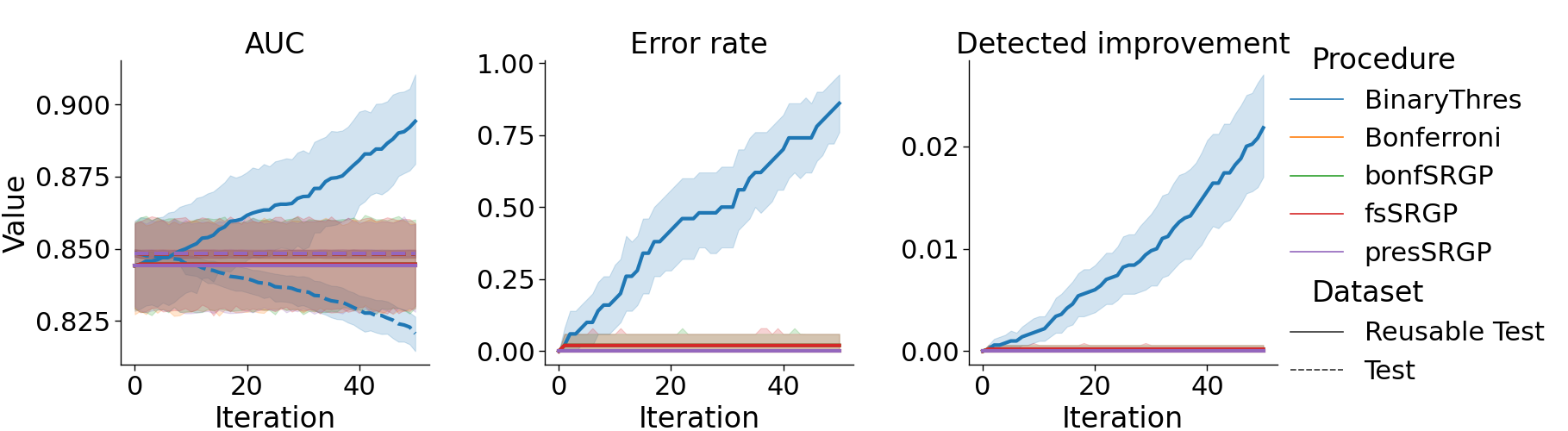}
\caption{
	Comparison of multiple testing procedures (MTPs) for approving algorithmic modifications, where the adaptive procedure tries to overfit to the reusable test dataset.
	Left: AUC of the most recently approved model on the reusable and a completely held out test dataset. Middle: the rate of incorrectly approving at least one unacceptable modification. Right: increase in the AUC detected by the MTPs.
}
\label{fig:fwer}
\vspace{-0.2cm}
\end{figure}

Here we show how MTPs that fail to control the FWER can drastically elevate one's risk of overfitting to the test data, as compared to appropriately-designed adaptive test data reuse procedures.
In particular, we consider the na\"ive procedure that tests every adaptive hypothesis at level $\alpha$ (\texttt{BinaryThres}).
The reusable test dataset has 100 observations and the model developer tests $T = 50$ modifications.
For the purpose of illustration, the initial model is set to the oracle, so all proposed modifications are unacceptable.

The simulated model developer tries to find models that overfit to the test data by searching within the neighborhood of the currently approved model.
In particular, the developer iteratively perturbs the coefficient of each irrelevant variable by 0.6 in the positive and negative directions.
When any such modification is approved, the model developer will continue perturbing that coefficient in the same direction until it fails to reject the null hypothesis.
For \texttt{presSRGP}, the prespecified model update at iteration $t$ is the model with coefficients exactly the same as the initial model except that the coefficient for the $(7 + \lfloor t/2\rfloor)$-th variable is set to 0.6 if $t$ is even and -0.6 if $t$ is odd.

Figure~\ref{fig:fwer} shows the result from 400 replicates.
Notably, \texttt{BinaryThres} approves at least one inferior modification with probability 75\% and concludes that the modifications by the last iteration improves the AUC by at least 0.02, \textit{even though the AUC actually drops by 0.025 on average}.
All the other MTPs appropriately control the FWER at the desired rate of 10\% and, thus, protect against over-fitting.

\vspace{-0.3cm}
\subsection{Assessing power}
\label{sec:sim_power}

Here the simulated model developer has access to an IID data stream and iteratively refits a logistic regression model on this data.
Because training on more data from the target population tends to improve model performance, the modifications are usually beneficial.
However, there is a risk that the modification does not improve performance or that the improvement is negligible, especially because there is a potential for overfitting to the reusable test data set.
By testing hypotheses \eqref{eq:null_delta}, we can restrict approval to only those model updates with meaningful improvements in the AUC.

The test dataset has 800 observations and we allow $T = 15$ adaptive tests.
At each time point, the model developer receives a new observation and refits the model.
To spend alpha more judiciously, the model developer will only submit the refitted model if the power calculations suggest that the probability for rejecting the null hypothesis exceeds 50\%.
Specifically, they perform power calculations by setting the true performance improvement to the CI lower bound, which is estimated using split-sample validation.
%perform simple-sample validation to construct a CI for the performance improvement and then calculate power by setting the true performance to the CI lower bound.
(For simplicity, the power calculations do not perform any multiple testing correction.)
To run \texttt{presSRGP}, the prespecified model updating procedure also selects updates based on a hypothesis test similar to \eqref{eq:null_delta} but replacing the adaptive difference sequence $\delta_j^{\adapt}$ with the prespecified difference sequence $\delta_j^{\prespec} = 0.0025 (j - 1)$ as well as replacing the adaptive modifications with the prespecified ones.

The procedures differed significantly in power (Figure~\ref{fig:power}).
On average, \texttt{Bonferroni} approved two modification, \texttt{BonfSRGP} approved 4.5, \texttt{fsSRGP} approved 5, and \texttt{presSRGP} approved 5.5.
By the end of the testing procedure, the average AUC of the approved model by \texttt{presSRGP} was 0.8 whereas \texttt{Bonferroni} only attained an AUC of 0.75.
Finally, the detected performance improvements was highest using \texttt{presSRGP}, as compared to the other methods.

\begin{figure}
	\includegraphics[width=0.5\textwidth]{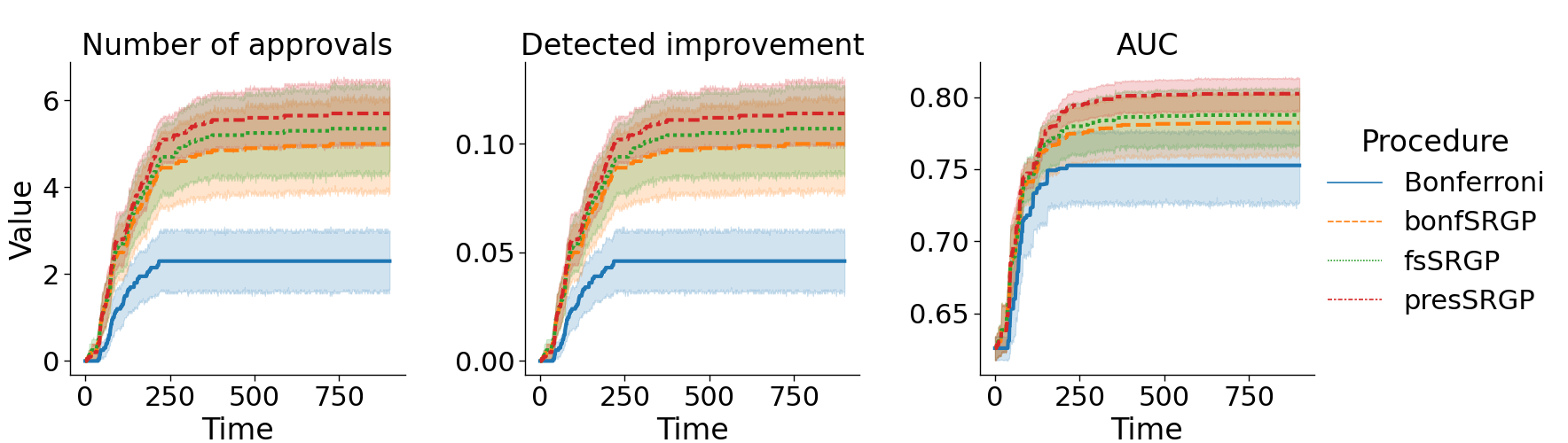}
	\caption{
		Comparison of multiple testing procedures (MTPs) for approving continually refitted models on a stream of IID data.
		Left: number of approved modifications.
		Middle: increase in AUC detected by MTPs.
		Right: AUC of the most recently approved modifications.
	}
	\label{fig:power}
\end{figure}

\vspace{-0.3cm}
\section{Revising predictions for acute hypotension episodes}
\label{sec:eci}

We now apply our procedure for approving modifications to a risk prediction model for acute hypotension episodes (AHEs), one of the most frequent critical events in the intensive care unit (ICU) \citep{Walsh2013-ur}.
The ICU is a clinical environment that continuously generates high throughput data. Thus, a model developer can readily collect new data in this setting to retrain an existing model.
To mimic this, we use data from the eICU Collaborative Research Database \citep{Pollard2018-mc}.
We train an initial model on 40 randomly selected admissions and simulate a data stream in which a randomly selected admission is observed at each time point.
The reusable test data is composed of 500 admissions and is used to evaluate $T=15$ modifications.

The task is to predict AHE 30 minutes in advance, where we define AHE as any 5-minute time period where the average mean arterial pressure (MAP) falls below 65 mmHg.
The input features to the model are baseline variables age, sex, height, and weight; vital signs MAP, heart rate, and respiration rate at the current time point; and the same set of vital signs five minutes prior.
The prediction model is a gradient boosted tree (GBT) and is continually refit on the incoming data.

Here we consider a more complex hypothesis test that checks for calibration-in-the-large \citep{Steyerberg2009-ze} \textit{and} improvement in AUC.
The $j$-th adaptive null hypothesis is
\begin{align}
	\begin{split}
H_{0,j}^{adapt}:
& \ \psi\left (\hat{f}^{adapt}_j; P_0 \right) \le \psi\left (\hat{f}_0; P_0 \right) + \delta^{adapt}_j \\
& \text{or } E\left[\hat{f}^{adapt}_j(X) - Y \right] \not\in [-\epsilon, \epsilon],
	\end{split}
\label{eq:eicu_hypo}
\end{align}
where $\delta^{adapt}_j$ is defined using the same procedure as that in Section~\ref{sec:sim_power}, $\hat{f}^{adapt}_j$ is the modification determined to have sufficient power for rejecting the null, and margin of error $\epsilon$ is 0.05.
We will refer to $E\left[\hat{f}^{adapt}_j(X) - Y \right]$ as calibration-error-in-the-large.
Details on calculating the test statistic and significance thresholds are provided in the Appendix.

Results from 40 replicates are shown in Figure~\ref{fig:eicu}.
We observe the same ranking of MTPs as that in Section~\ref{sec:sim_power}: \texttt{prespecSRGP} performed the best, followed by \texttt{fsSRGP}.
Compared to the previous section, the relative improvement between the methods is smaller because the GBTs improved rapidly at early time points and slowed down thereafter.

\begin{figure}
	\centering
	\includegraphics[width=0.4\textwidth]{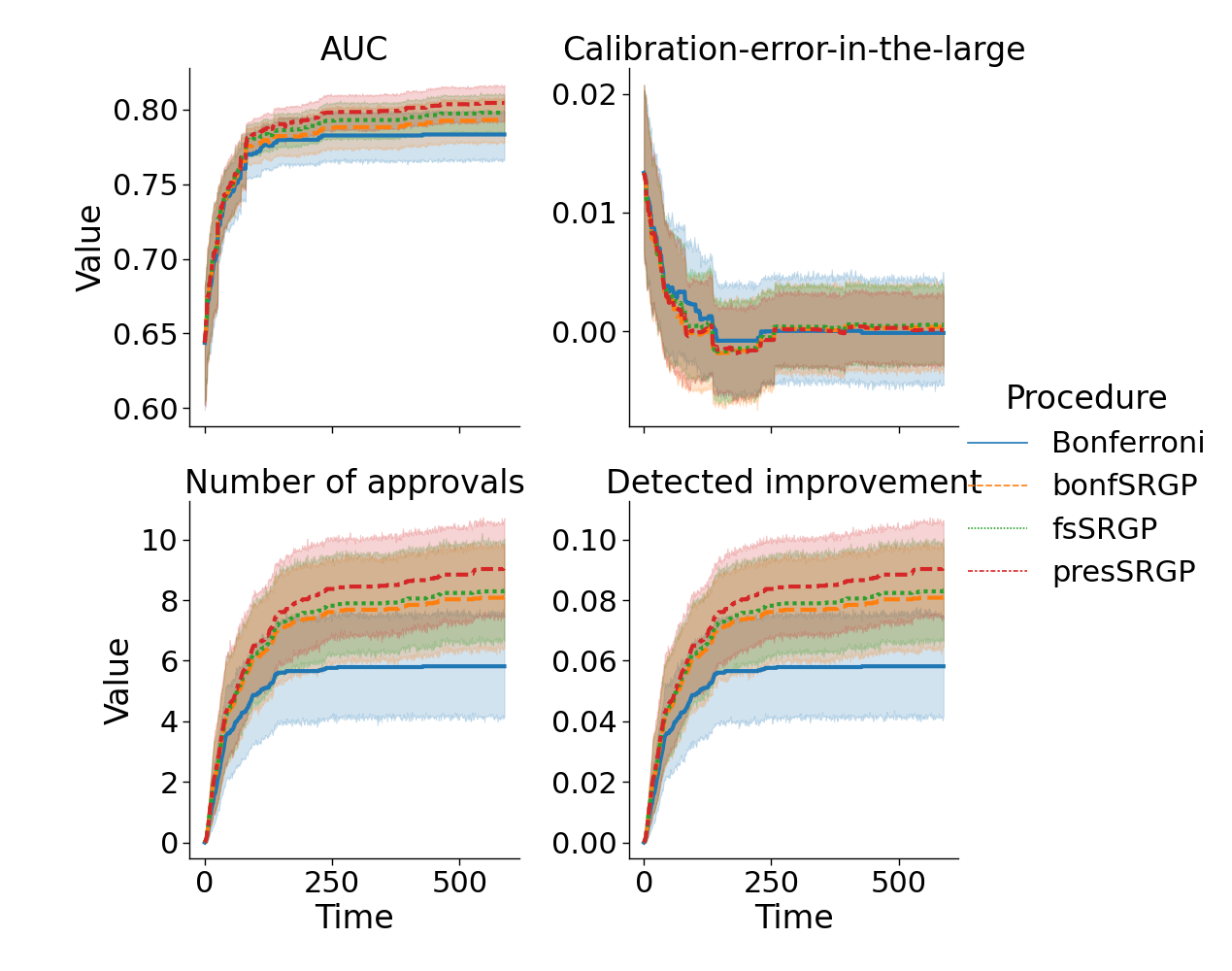}
	\vspace{-0.25cm}
	\caption{Approving refitted gradient boosted trees for predicting acute hypotension episodes (AHEs), checking that the calibration-error-in-the-large is close to the ideal value of zero and that the AUC is improving.}
	\label{fig:eicu}
\end{figure}

\vspace{-0.3cm}
\section{Related Work}
Our paper relates to a large body of work on methods for providing valid statistical inference and preventing false discoveries.
Much of this literature has focused on testing prespecified hypotheses on the same dataset while controlling the FWER \citep{Hochberg1987-db, Westfall2010-it}, false discovery rate (FDR) \citep{Benjamini1995-xt}, or some variant thereof \citep{Van_der_Laan2004-qz}.
More recent works consider testing a sequence of adaptive hypotheses on prospectively-collected data from a data stream and controlling online error rates \citep{Foster2008-ek, Ramdas2018-gh}.
This work considers the setting where we adaptively test hypotheses on the \textit{same} dataset.
To control the bias, testing procedures must limit the amount of information released about the test dataset \citep{Russo2016-ms}.
Techniques based on differential privacy, which is a mathematically rigorous formalization of data privacy \citep{Dwork2014-ot}, do this by adding random noise (e.g. Laplace or Gaussian noise) to the test statistic or, more generally, the queried result \citep{Dwork2015-da}.
While theoretical guarantees are available for differential privacy based methods for test data reuse (e.g., \citep{Dwork2015-ho, Russo2016-ms, Rogers2016-mt, Cummings2016-hu, Dwork2017-zf, Feldman2017-xi, Feldman2018-eo, Shenfeld2019-gq, Gossmann2021-qk} and others), the required size of the test dataset is prohibitively large for many application domains or require injecting very large amounts of noise \citep{Rogers2019-br, Gossmann2021-qk}.
% When the test datasets are small, these approaches must inject a very large amount of noise to protect against overfitting to the test set with theoretical guarantees.
An alternative approach is to directly limit the number of bits of information released to the model developer by discretizing the queried result along some grid \citep{Blum2015-hv}.
Existing methods essentially perform a Bonferroni correction for the number of distinct hypotheses, which also require unreasonably large test datasets for many applications.
To improve testing power, a number of works have assumed that the adaptivity of the model developer is limited (e.g. the models are highly correlated, or the model developer is not entirely ``adversarial'') to justify the use of a less conservative correction factor \citep{Mania2019-pt, Zrnic2019-lt}.
In contrast, the SRGPs proposed in this work achieve higher power via alpha-recycling and account for the correlation structure without needing to make assumptions about the model developer.
\vspace{-0.4cm}
\section{Conclusion}
\vspace{-0.1cm}

We show how to leverage SRGPs to design valid and powerful approaches for testing a sequence of adaptively-defined algorithmic modifications on the same dataset.
The overall steps of this framework are (i) limit the amount of information leakage by reporting only binary test results (approve versus deny modifications), (ii) spend \textit{and recycle} alpha using an SRGP, and (iii) design consonant, closed-testing procedures whose significance thresholds can be computed without needing to observe the counterfactual hypotheses.
To account for correlation between the algorithmic modifications, we presented two new SRGPs.
\texttt{fsSRGP} achieves higher power by leveraging the correlation structure between the observed algorithmic modifications.
\texttt{presSRGP} asks the model developer to generate a sequence of algorithmic modifications using a prespecified learning procedure and leverages the correlation between the adaptive and prespecified algorithmic modifications.
In empirical studies, these procedures approved more algorithmic modifications than existing methods, with \texttt{presSRGP} achieving the highest power.

One direction of future work is to optimize the power for approving algorithmic modifications by (i) tuning the node and edge weights in the SRGP and (ii) exploring various testing strategies that the model developer can employ.
In addition, model developers are often interested in obtaining more detailed test results like p-values and confidence intervals.
Another direction for future work is to design SRGPs that release more information per iteration, perhaps by leveraging differential privacy techniques.

\begin{acknowledgements}
We thank Noah Simon, Charles McCulloch, and Zhenghao Chen for helpful discussions and suggestions.
We are grateful to Nicholas Fong for sharing cleaned data for the acute hypotension episode example.

This  work  was  supported  by  the  Food  and  Drug  Administration (FDA)  of  the  U.S.  Department  of  Health  and  Human  Services (HHS) as part of a financial assistance award Center of Excellence in Regulatory Science and Innovation grant to University of California, San Francisco (UCSF) and Stanford University, U01FD005978 totaling \$79,250 with 100\% funded by FDA/HHS. The contents are those of the author(s) and do not necessarily represent the official views of, nor an endorsement, by FDA/HHS, or the U.S. Government.

\end{acknowledgements}

\bibliography{notes}

\begin{thebibliography}{38}
\providecommand{\natexlab}[1]{#1}
\providecommand{\url}[1]{\texttt{#1}}
\expandafter\ifx\csname urlstyle\endcsname\relax
  \providecommand{\doi}[1]{doi: #1}\else
  \providecommand{\doi}{doi: \begingroup \urlstyle{rm}\Url}\fi

\bibitem[Benjamini and Hochberg(1995)]{Benjamini1995-xt}
Yoav Benjamini and Yosef Hochberg.
\newblock Controlling the false discovery rate: A practical and powerful
  approach to multiple testing.
\newblock \emph{J. R. Stat. Soc. Series B Stat. Methodol.}, 57\penalty0
  (1):\penalty0 289--300, 1995.

\bibitem[Blum and Hardt(2015)]{Blum2015-hv}
Avrim Blum and Moritz Hardt.
\newblock The ladder: A reliable leaderboard for machine learning competitions.
\newblock \emph{International Conference on Machine Learning}, 37:\penalty0
  1006--1014, 2015.

\bibitem[Bretz et~al.(2009)Bretz, Maurer, Brannath, and Posch]{Bretz2009-bt}
Frank Bretz, Willi Maurer, Werner Brannath, and Martin Posch.
\newblock A graphical approach to sequentially rejective multiple test
  procedures.
\newblock \emph{Stat. Med.}, 28\penalty0 (4):\penalty0 586--604, February 2009.

\bibitem[Bretz et~al.(2011{\natexlab{a}})Bretz, Maurer, and
  Hommel]{Bretz2011-hd}
Frank Bretz, Willi Maurer, and Gerhard Hommel.
\newblock Test and power considerations for multiple endpoint analyses using
  sequentially rejective graphical procedures.
\newblock \emph{Stat. Med.}, 30\penalty0 (13):\penalty0 1489--1501, June
  2011{\natexlab{a}}.

\bibitem[Bretz et~al.(2011{\natexlab{b}})Bretz, Posch, Glimm, Klinglmueller,
  Maurer, and Rohmeyer]{Bretz2011-mu}
Frank Bretz, Martin Posch, Ekkehard Glimm, Florian Klinglmueller, Willi Maurer,
  and Kornelius Rohmeyer.
\newblock Graphical approaches for multiple comparison procedures using
  weighted bonferroni, simes, or parametric tests.
\newblock \emph{Biom. J.}, 53\penalty0 (6):\penalty0 894--913, November
  2011{\natexlab{b}}.

\bibitem[Burman et~al.(2009)Burman, Sonesson, and Guilbaud]{Burman2009-pp}
C-F Burman, C~Sonesson, and O~Guilbaud.
\newblock A recycling framework for the construction of bonferroni-based
  multiple tests.
\newblock \emph{Stat. Med.}, 28\penalty0 (5):\penalty0 739--761, February 2009.

\bibitem[Cummings et~al.(2016)Cummings, Ligett, Nissim, Roth, and
  Wu]{Cummings2016-hu}
Rachel Cummings, Katrina Ligett, Kobbi Nissim, Aaron Roth, and Zhiwei~Steven
  Wu.
\newblock {Adaptive Learning with Robust Generalization Guarantees}.
\newblock In \emph{{Conference on Learning Theory}}, pages 772--814, June 2016.

\bibitem[Dwork and Roth(2014)]{Dwork2014-ot}
Cynthia Dwork and Aaron Roth.
\newblock The algorithmic foundations of differential privacy.
\newblock \emph{Found. Trends Theor. Comput. Sci.}, 9\penalty0 (3--4):\penalty0
  211--407, 2014.

\bibitem[Dwork et~al.(2015{\natexlab{a}})Dwork, Feldman, Hardt, Pitassi,
  Reingold, and Roth]{Dwork2015-da}
Cynthia Dwork, Vitaly Feldman, Moritz Hardt, Toniann Pitassi, Omer Reingold,
  and Aaron Roth.
\newblock The reusable holdout: Preserving validity in adaptive data analysis.
\newblock \emph{Science}, 349\penalty0 (6248):\penalty0 636--638, August
  2015{\natexlab{a}}.

\bibitem[Dwork et~al.(2015{\natexlab{b}})Dwork, Su, and Zhang]{Dwork2015-ho}
Cynthia Dwork, Weijie Su, and Li~Zhang.
\newblock Private false discovery rate control.
\newblock November 2015{\natexlab{b}}.

\bibitem[Dwork et~al.(2017)Dwork, Feldman, Hardt, Pitassi, Reingold, and
  Roth]{Dwork2017-zf}
Cynthia Dwork, Vitaly Feldman, Moritz Hardt, Toniann Pitassi, Omer Reingold,
  and Aaron Roth.
\newblock {Guilt-free Data Reuse}.
\newblock \emph{Communications of the ACM}, 60\penalty0 (4):\penalty0 86--93,
  March 2017.
\newblock ISSN 0001-0782.

\bibitem[Feldman and Steinke(2017)]{Feldman2017-xi}
Vitaly Feldman and Thomas Steinke.
\newblock Generalization for adaptively-chosen estimators via stable median.
\newblock In \emph{Proceedings of the 2017 Conference on Learning Theory},
  volume~65, pages 728--757. PMLR, 2017.

\bibitem[Feldman and Steinke(2018)]{Feldman2018-eo}
Vitaly Feldman and Thomas Steinke.
\newblock Calibrating noise to variance in adaptive data analysis.
\newblock In S{\'e}bastien Bubeck, Vianney Perchet, and Philippe Rigollet,
  editors, \emph{Proceedings of the 31st Conference On Learning Theory},
  volume~75 of \emph{Proceedings of Machine Learning Research}, pages 535--544.
  PMLR, 2018.

\bibitem[Feng et~al.(2020)Feng, Emerson, and Simon]{Feng2020-ev}
Jean Feng, Scott Emerson, and Noah Simon.
\newblock Approval policies for modifications to machine learning-based
  software as a medical device: a study of bio-creep.
\newblock \emph{Biometrics}, September 2020.

\bibitem[Foster and Stine(2008)]{Foster2008-ek}
Dean~P Foster and Robert~A Stine.
\newblock $\alpha$-investing: a procedure for sequential control of expected
  false discoveries.
\newblock \emph{J. R. Stat. Soc. Series B Stat. Methodol.}, 70\penalty0
  (2):\penalty0 429--444, April 2008.

\bibitem[Gabriel(1969)]{Gabriel1969-mk}
K~R Gabriel.
\newblock Simultaneous test {Procedures--Some} theory of multiple comparisons.
\newblock \emph{aoms}, 40\penalty0 (1):\penalty0 224--250, February 1969.

\bibitem[Gelman and Loken(2017)]{Gelman2017-kh}
Andrew Gelman and Eric Loken.
\newblock The statistical crisis in science.
\newblock \emph{American Scientist}, February 2017.

\bibitem[Gossmann et~al.(2021)Gossmann, Pezeshk, Wang, and
  Sahiner]{Gossmann2021-qk}
Alexej Gossmann, Aria Pezeshk, Yu-Ping Wang, and Berkman Sahiner.
\newblock Test data reuse for the evaluation of continuously evolving
  classification algorithms using the area under the receiver operating
  characteristic curve.
\newblock \emph{SIAM Journal on Mathematics of Data Science}, pages 692--714,
  January 2021.

\bibitem[Hochberg and Tamhane(1987)]{Hochberg1987-db}
Yosef Hochberg and Ajit Tamhane.
\newblock \emph{Multiple comparison procedures}.
\newblock Wiley, New York, 1987.

\bibitem[Hommel et~al.(2007)Hommel, Bretz, and Maurer]{Hommel2007-mw}
Gerhard Hommel, Frank Bretz, and Willi Maurer.
\newblock Powerful short-cuts for multiple testing procedures with special
  reference to gatekeeping strategies.
\newblock \emph{Stat. Med.}, 26\penalty0 (22):\penalty0 4063--4073, September
  2007.

\bibitem[LeDell et~al.(2015)LeDell, Petersen, and van~der Laan]{LeDell2015-qz}
Erin LeDell, Maya Petersen, and Mark van~der Laan.
\newblock Computationally efficient confidence intervals for cross-validated
  area under the {ROC} curve estimates.
\newblock \emph{Electron. J. Stat.}, 9\penalty0 (1):\penalty0 1583--1607, 2015.

\bibitem[Lehmann and Romano(2005)]{Lehmann2005-us}
E~L Lehmann and Joseph~P Romano.
\newblock \emph{Testing Statistical Hypotheses}.
\newblock Springer, New York, NY, 2005.

\bibitem[Mania et~al.(2019)Mania, Miller, Schmidt, Hardt, and
  Recht]{Mania2019-pt}
Horia Mania, John Miller, Ludwig Schmidt, Moritz Hardt, and Benjamin Recht.
\newblock Model similarity mitigates test set overuse.
\newblock In \emph{Advances in Neural Information Processing Systems},
  volume~32. Curran Associates, Inc., 2019.

\bibitem[Millen and Dmitrienko(2011)]{Millen2011-pw}
Brian~A Millen and Alex Dmitrienko.
\newblock Chain procedures: A class of flexible closed testing procedures with
  clinical trial applications.
\newblock \emph{Stat. Biopharm. Res.}, 3\penalty0 (1):\penalty0 14--30,
  February 2011.

\bibitem[Pollard et~al.(2018)Pollard, Johnson, Raffa, Celi, Mark, and
  Badawi]{Pollard2018-mc}
Tom~J Pollard, Alistair E~W Johnson, Jesse~D Raffa, Leo~A Celi, Roger~G Mark,
  and Omar Badawi.
\newblock The {eICU} collaborative research database, a freely available
  multi-center database for critical care research.
\newblock \emph{Sci Data}, 5:\penalty0 180178, September 2018.

\bibitem[Ramdas et~al.(2018)Ramdas, Zrnic, Wainwright, and
  Jordan]{Ramdas2018-gh}
Aaditya Ramdas, Tijana Zrnic, Martin Wainwright, and Michael Jordan.
\newblock {SAFFRON}: an adaptive algorithm for online control of the false
  discovery rate.
\newblock \emph{International Conference on Machine Learning}, 2018.

\bibitem[Rogers et~al.(2016)Rogers, Roth, Smith, and Thakkar]{Rogers2016-mt}
R~Rogers, A~Roth, A~Smith, and O~Thakkar.
\newblock {Max-Information, Differential Privacy, and Post-selection Hypothesis
  Testing}.
\newblock In \emph{{2016 IEEE 57th Annual Symposium on Foundations of Computer
  Science (FOCS)}}, pages 487--494, October 2016.

\bibitem[Rogers et~al.(2019)Rogers, Roth, Smith, Srebro, Thakkar, and
  Woodworth]{Rogers2019-br}
Ryan Rogers, Aaron Roth, Adam Smith, Nathan Srebro, Om~Thakkar, and Blake
  Woodworth.
\newblock Guaranteed validity for empirical approaches to adaptive data
  analysis.
\newblock June 2019.

\bibitem[Russo and Zou(2016)]{Russo2016-ms}
Daniel Russo and James Zou.
\newblock Controlling bias in adaptive data analysis using information theory.
\newblock \emph{International Conference on Artificial Intelligence and
  Statistics}, 51:\penalty0 1232--1240, 2016.

\bibitem[Shenfeld and Ligett(2019)]{Shenfeld2019-gq}
Moshe Shenfeld and Katrina Ligett.
\newblock {A Necessary and Sufficient Stability Notion for Adaptive
  Generalization}.
\newblock In \emph{{Advances in Neural Information Processing Systems}}, pages
  11481--11490, 2019.

\bibitem[Steyerberg(2009)]{Steyerberg2009-ze}
Ewout~W Steyerberg.
\newblock \emph{Clinical Prediction Models: A Practical Approach to
  Development, Validation, and Updating}.
\newblock Springer, New York, NY, 2009.

\bibitem[Thompson et~al.(2020)Thompson, Wright, Bissett, and
  Poldrack]{Thompson2020-ut}
William~Hedley Thompson, Jessey Wright, Patrick~G Bissett, and Russell~A
  Poldrack.
\newblock Dataset decay and the problem of sequential analyses on open
  datasets.
\newblock \emph{Elife}, 9, May 2020.

\bibitem[{U.S. FDA}(2019)]{US_Food_and_Drug_Administration2019-kt}
{U.S. FDA}.
\newblock Proposed regulatory framework for modifications to artificial
  intelligence/machine learning ({AI/ML)-based} software as a medical device
  ({SaMD)}: discussion paper and request for feedback.
\newblock Technical report, April 2019.

\bibitem[van~der Laan et~al.(2004)van~der Laan, Dudoit, and
  Pollard]{Van_der_Laan2004-qz}
Mark~J van~der Laan, Sandrine Dudoit, and Katherine~S Pollard.
\newblock Augmentation procedures for control of the generalized family-wise
  error rate and tail probabilities for the proportion of false positives.
\newblock \emph{Stat. Appl. Genet. Mol. Biol.}, 3\penalty0 (1):\penalty0
  Article15, June 2004.

\bibitem[Walsh et~al.(2013)Walsh, Devereaux, Garg, Kurz, Turan, Rodseth,
  Cywinski, Thabane, and Sessler]{Walsh2013-ur}
Michael Walsh, Philip~J Devereaux, Amit~X Garg, Andrea Kurz, Alparslan Turan,
  Reitze~N Rodseth, Jacek Cywinski, Lehana Thabane, and Daniel~I Sessler.
\newblock Relationship between intraoperative mean arterial pressure and
  clinical outcomes after noncardiac surgery: toward an empirical definition of
  hypotension.
\newblock \emph{Anesthesiology}, 119\penalty0 (3):\penalty0 507--515, September
  2013.

\bibitem[Westfall and Stanley~Young(1993)]{Westfall1993-ac}
Peter~H Westfall and S~Stanley~Young.
\newblock \emph{{Resampling-Based} Multiple Testing: Examples and Methods for
  p-Value Adjustment}.
\newblock John Wiley \& Sons, January 1993.

\bibitem[Westfall et~al.(2010)Westfall, Troendle, and
  Pennello]{Westfall2010-it}
Peter~H Westfall, James~F Troendle, and Gene Pennello.
\newblock Multiple {McNemar} tests.
\newblock \emph{Biometrics}, 66\penalty0 (4):\penalty0 1185--1191, December
  2010.

\bibitem[Zrnic and Hardt(2019)]{Zrnic2019-lt}
T~Zrnic and M~Hardt.
\newblock Natural analysts in adaptive data analysis.
\newblock \emph{International Conference on Machine}, 2019.

\end{thebibliography}

\newpage

\onecolumn

\appendix

\section{Proofs}

\begin{lemma}
	The adaptive SRGP in Algorithm~\ref{algo:graph_update} with a fixed strategy is equivalent to a prespecified SRGP.
	\label{lemma:equiv_prespec}
\end{lemma}
\begin{proof}
	We define a filtration over approval histories up to the maximum number of iterations $T$.
	That is, define sample space $\Omega$ as the set of approval histories over $T$ iterations, i.e. $\Omega = \{0,1\}^{T - 1}$, and $\sigma$-algebras $\mathcal{F}_t$ for $t = 1,...,T$ over approval histories up to iteration $t$.
	To show that the adaptive SRGP is equivalent to a prespecified SRGP, we need to show that that the adaptive procedure defines a set of hypotheses, node weights, and edge weights for the initial set of hypotheses $I_0$, the hypotheses and weights are $\mathcal{F}_1$-measurable functions, and the weight constraints are satisfied.
	First, we note that the edge weights being elicited at iteration $t$ in Algorithm~\ref{algo:graph_update} is equivalent to eliciting the edge weights for the initial set of hypotheses $I_0$, i.e. $g_{a_{t'}, a_t} = g_{a_{t'}, a_t}(I_0)$ in Algorithm~\ref{algo:graph_update}.
	This is because we only elicit the edge weight $g_{a_{\tau_t}, a_t}$ if there has been no approval since time $\tau_t$ so the edge weights being elicited are never updated via the edge-weight renormalization step in SRGPs.
	As such, the adaptive SRGP in Algorithm~\ref{algo:graph_update} for a model developer with a fixed strategy for selecting hypotheses and weights can be described to have a fixed hypothesis testing tree structure with
	\begin{itemize}
		\item $\mathcal{F}_t$-measurable hypotheses $H_{a_{t}}(I_0)$ for all $a_{t}$
		\item $\mathcal{F}_1$-measurable node weights $w_{a_t}(I_0)$ for all $a_t \in \{0,1\}^{T - 1}$ that satisfy the constraint that they sum to one,
		\item and $\mathcal{F}_t$-measurable edge weights $g_{a_{t'}, a_t}(I_0)$ for all valid edges $(a_{t'}, a_t)$ in the graph that satisfy the constraint that all outgoing edge weights sum to one.
	\end{itemize}
	Although the hypotheses and edge weights are $\mathcal{F}_t$-measurable, they can also be viewed as $\mathcal{F}_1$-measurable functions over the input space $a_t$ and $(a_{t'}, a_t)$, respectively.
	Moreover, the edge weights satisfy the edge weight constraints by design.
	Thus the adaptive SRGP satisfies the node and edge weights constraints with respect to $\mathcal{F}_1$.

\end{proof}

\begin{lemma}
	If the adaptive SRGP in Algorithm~\ref{algo:graph_update} controls the FWER for any fixed strategy, then the adaptive SRGP in Algorithm~\ref{algo:graph_update} controls the FWER for any stochastic strategy.
	\label{lemma:stochastic}
\end{lemma}

\begin{proof}
	Let $\mathcal{S}$ be the set of all fixed strategies. The stochastic adaptive strategy is a random distribution over $\mathcal{S}$.
	Its FWER is
	\begin{align*}
		\Pr\left(\text{incorrectly reject some } H_{t}^{\adapt} \right)
		= \sum_{s \in \mathcal{S}} \Pr(S = s) \Pr\left(\text{incorrectly reject some } H_{t}^{\adapt} \mid S = s \right)
	\end{align*}
	where the latter probability on the right hand side is the FWER for a fixed strategy $s$.
	As such, the FWER of the stochastic strategy is properly controlled as long as the FWER of any fixed strategy is properly controlled.
\end{proof}

\begin{corollary}
	Algorithm~\ref{algo:graph_update} with the significance thresholds defined per
	\begin{equation}
		 c_{a_{t}}(I_{t}) = w_{a_{t}}(I_{t}) \alpha
		 \label{eq:srgp_bonf_crit}
	\end{equation}
	controls the FWER at level $\alpha$.
\end{corollary}
\begin{proof}
	Per Lemmas~\ref{lemma:equiv_prespec} and \ref{lemma:stochastic}, it suffices to show that the fully prespecified SRGP controls the FWER.
	Recall that \eqref{eq:srgp_bonf_crit} is a closed weighted Bonferroni test in \citet{Bretz2011-hd}.
	As such, any fixed or stochastic adaptive strategy would control FWER.
\end{proof}

\begin{proof}[Proof for Theorem~\ref{thrm:ffs}]
	Per Lemmas~\ref{lemma:equiv_prespec} and \ref{lemma:stochastic}, it suffices to show that the fully prespecified SRGP controls the FWER.

	First, per the proof in \citet{Bretz2009-bt}, we note that node weights for any intersection hypothesis $I$ calculated using Algorithm~\ref{algo:graph_update} are well-defined, in that it does not depend on ordering in which we remove nodes from the graph.

	We begin with proving that for any intersection hypothesis $I$, the critical values calculated using \eqref{eq:sig_thres_corr} controls the Type I error.
	First we show that for any $a_t$ ending with success (i.e. $a_{t, t - 1} = 1$) and any $I$, the calculated critical values for testing the intersection hypotheses $\cap_{a_k \in G_{a_t} \cap I} H_{a_k}$ controls the Type I error at level $\left ( \sum_{a_{k} \in G_{a_t} \cap I} w_{a_k}(I)\right) \alpha$.
	% 	That is,
	% 	\begin{align}
		% 	\Pr\left(p_{k} > c_{k} \forall k \in K, p_{j} < \tilde{c} | \cap_{k \in K} H_{k} \cap H_j \right)
		% 	\le \frac{w_{j}}{\sum_{k\in K}w_{k} + \sum_{j'=j}^\infty w_{j'}} \alpha
		% 	\qquad
		% 	\forall K \subseteq\{1,...,j-1\}
		% 	\end{align}
	% 	Suppose hypotheses with indices $I$ are true.
	% 	Then
	% The critical values for rejecting hypotheses in $I$ assigned via $G$ are strictly smaller than the crit values for each element in $I$ if we knew $I$.
	% Under $I$, we have
	Per the definition of the critical values in \eqref{eq:sig_thres_corr}, we have that
	\begin{align*}
		& \Pr \left (\text{we reject for some } a_j \in G_{a_t} \cap I |
		H_{G_{a_t} \cap I}
		\right)\\
		= &
		\sum_{a_{j} \in (G_{a_t} \cap I)} \Pr\left(p_{a_k} > c_{a_k}(I) \forall a_{k} \in G_{a_t} \cap I, k< j, p_{a_j} < c_{a_j}(I) | \cap_{a_{k} \in G_{a_t} \cap I, k\le j} H_{a_{k}} \right)\\
		\le & \left ( \sum_{a_{j} \in (G_{a_t} \cap I)} w_{a_j}(I)\right) \alpha.
	\end{align*}
	Therefore, as long as the total node weight across $I$ is no more than one, we control the Type I error at level $\alpha$.
	Because Type I error control holds for all intersection hypotheses $I$, we have established that this procedure is a valid closed test.

	Next, per the proof  in \citet{Bretz2009-bt}, we must show that the critical values satisfy the monotonicity condition to prove that our procedure is a valid consonant, shortcut procedure.
	More specifically, we require the following to hold for all $t = 1,...,T$:
	\begin{equation}
		c_{a_t}(I) < c_{a_t}(J) \qquad \forall J\subseteq I.
		\label{eq:monotonic_spec}
	\end{equation}
	The proof is by induction.
	It is easy to see that \eqref{eq:monotonic_spec} holds for $t = 1$.
	Suppose \eqref{eq:monotonic_spec} holds for $1,...,t - 1$.
	Now consider any history $a_{\tilde{t}}$ that ends with an approval.
	Consider any $a_t$ and subset $J \subseteq I$ such that $a_t \in G_{a_{\tilde{t}}} \cap J$.
	We have that
	\begin{align*}
		& c_{a_t}(J)\\
		=& \sup\left\{
		\tilde{c}:
		\Pr\left(p_{a_k} > c_{a_k}(J) \forall a_k \in K, p_{t} < \tilde{c} | H_{K \cup \{a_t\}} \right)
		\le \left[\sum_{\substack{a_k \in ((G_{a_{\tilde{t}}} \cap J) \setminus K) \\ k \le t}}
		w_{a_k} (J)\right] \alpha
		\forall K \subseteq \{a_k: a_k \in G_{a_{\tilde{t}}} \cap J, k < t \}
		\right\}\\
		\ge & \sup\left\{
		\tilde{c}:
		\Pr\left(p_{a_k} > c_{a_k}(I) \forall a_k \in K, p_{t} < \tilde{c} | H_{K \cup \{a_t\}} \right)
		\le \left[
		\sum_{\substack{a_k \in ((G_{a_{\tilde{t}}} \cap J) \setminus K) \\ k \le t}}
		w_{a_k} (J)\right] \alpha
		\forall K \subseteq \{a_k: a_k \in G_{a_{\tilde{t}}} \cap J, k < t \}
		\right\}\\
		\ge&  \sup\left\{
		\tilde{c}:
		\Pr\left(p_{a_k} > c_{a_k}(I) \forall a_k \in K, p_{t} < \tilde{c} | H_{K \cup \{a_t\}} \right)
		\le \left[
		\sum_{\substack{a_k \in ((G_{a_{\tilde{t}}} \cap I) \setminus K) \\ k \le t}}
		w_{a_k} (I)\right] \alpha
		\forall
		K \subseteq \{a_k: a_k \in G_{a_{\tilde{t}}} \cap I, k < t \}
		\right\}\\
		= &\   c_{a_t}(I)
	\end{align*}
	where the first inequality follows by induction and the second inequality is because the weights are monotonic.
\end{proof}

\begin{proof}[Proof for Theorem~\ref{thrm:parallel}]
	Per Lemmas~\ref{lemma:equiv_prespec} and \ref{lemma:stochastic}, it suffices to show that the fully prespecified SRGP controls the FWER.

	We first prove that the critical values per \eqref{eq:adapt_err} control the Type I error for any intersection hypothesis $I$.
	For any $I$, define $\tilde{I}$ as the union of $I$ and all prespecified nodes.
	Then the Type I error can be bounded using a sequence of union bounds:
	\begin{align*}
		& \Pr\left(
		\exists (t, a_t) \in I \text{ s.t. } p_{a_t} < c_{a_t}({I}) \mid H_{I}
		\right)\\
		\le & \Pr\left(
		\exists t \text{ s.t. } \xi_{t,n}^{\prespec} \le z_{t}^{\prespec}({I}) \text{ OR }  \exists (t, a_t) \in I \text{ s.t. } p_{a_t} < c_{a_t}({I}) \mid H_{I}
		\right)\\
		\le &
		\sum_{t = 1}^\infty
		\left[
		\Pr\left(
		\xi_{t',n}^{\prespec} > z_{t'}^{\prespec}({I}) \forall t' \le t - 1,
		\xi_{t,n}^{\prespec} \le z_{t}^{\prespec}({I}) \mid H_{I}
		\right)
		+
		\sum_{a_t \in I}
		\Pr\left(
		\xi_{t',n}^{\prespec} > z_{t'}^{\prespec}({I}) \forall t' \le t,
		p_{a_t} < c_{a_t}({I})\mid H_{I}
		\right)\right]\\
		\le &
		\left (\sum_{t = 1}^\infty \left(w_{t}^{\prespec}\left ( \tilde{I}\right)
		+ \sum_{a_t \in I} w_{a_t}\left(\tilde{I}\right ) \right) \right) \alpha\\
		=&\alpha.
	\end{align*}
	% 	\red{[AG: In the currently first line of the equation shouldnt it be\\ $\Pr\left(
		% 	\exists t \text{ s.t. } \xi_{t,n}^{\prespec} \le z_{t}^{\prespec}({I}) \text{ OR } \left[ (\forall t): \xi_{t,n}^{\prespec} > z_{t}^{\prespec}({I}) \text{AND} \exists (t, a_t) \in I \text{ s.t. } p_{a_t} < c_{a_t}({I}) \right] \mid H_{I}
		% 	\right)$ ? (maybe I'm misunderstanding something)]}\\
	Because the weights are nondecreasing in Algorithm~\ref{algo:graph_update}, the critical values defined in \eqref{eq:adapt_err} satisfy the monotonicity condition.
	As such, Algorithm~\ref{algo:graph_update} is a consonant, short-cut procedure for the above closed test.

\end{proof}

\section{Hypothesis test details}

\subsection{Testing for an improvement in AUC}
\label{sec:auc_compare}

In Section~\ref{sec:sim}, we decide whether or not to approve a modification by testing the adaptively-defined null hypothesis \eqref{eq:null_delta} at each iteration $j$, which compares the AUC between the $j$th adaptively proposed model and the initial model.
Per Algorithm~\ref{algo:graph_update}, we test the adaptive hypotheses by treating them as pre-specified hypotheses from a bifurcating tree, i.e.
\begin{equation}
	H_{0,a_j}: \psi\left (\hat{f}_{a_j}, P_0 \right) \le \psi\left (\hat{f}_0; P_0 \right) + \delta_{a_j}
	\label{eq:null_delta_fix}
\end{equation}
for approval histories $a_j$.
We now describe how the test statistics and significance thresholds are constructed.

Recall that the AUC is equal to the Mann-Whitney U-statistic for comparing ranks across two populations, i.e.
\begin{align}
	\psi(f, P_0) = P_0 \left (f(X_1) > f(X_2) \mid Y_1 = 1, Y_2 = 0 \right),
\end{align}
where $(X_1,Y_1)$ and $(X_2,Y_2)$ represent independent draws from $P_0$.
The empirical AUC is defined as
\begin{equation}
	\psi(f, P_n) = \frac{1}{n_0 n_1} \sum_{i=1}^{n_0} \sum_{j=1}^{n_1} \mathbbm{1}\left\{f(X_j) > f(X_i)\right\} \mathbbm{1}\left\{Y_j = 1, Y_i = 0\right\},
	\label{eq:auc_emp}
\end{equation}
where $n_0$ is the number of observations with $Y = 0$ and $n_1 = n - n_0$.
To test \eqref{eq:null_delta_fix}, we characterize the asymptotic distribution of \eqref{eq:auc_emp} by analyzing its influence function.
Given IID observations from $P_0$, \eqref{eq:auc_emp} is an asymptotically linear estimator of the model's AUC \citep{LeDell2015-qz}, in that
\begin{equation}
	\psi(f, P_n) - \psi(f, P_0)
	= \frac{1}{n} \sum_{i=1}^n \phi(f, P_0)(X_i, Y_i) + o_p(1/\sqrt{n})
	\label{eq:auc_diff}
\end{equation}
with influence function
\begin{align*}
	\begin{split}
	\phi(f, P_0)(X_i, Y_i) =&
	\frac{\mathbbm{1}\{Y_i = 1\}}{P_0(Y = 1)} P_0\left(f(X) < c \mid Y= 0; c = f(X_i)\right)\\
	&\ + \frac{\mathbbm{1}\{Y_i = 0\}}{P_0(Y = 0)} P_0\left(f(X) > c \mid Y= 1; c = f(X_i)\right)\\
	&\ - \left\{
	\frac{\mathbbm{1}\{Y_i = 0\}}{P_0(Y = 0)}
	+ \frac{\mathbbm{1}\{Y_i = 0\}}{P_0(Y = 0)}
	\right\} \psi\left(f, P_0\right).
	\end{split}
\end{align*}
Per the Central Limit Theorem, we have that
\begin{equation}
\sqrt{n}\left(\psi(f, P_n) - \psi(f, P_0) \right) \rightarrow_d N\left(0, \sigma(f,P_0)^2 \right)
\end{equation}
where $\sigma(f,P_0)^2 = \Var(\phi(f, P_0)(X,Y))$.
We can then test the null hypothesis $H_0: \psi(\hat{f}_0, P_0) \le c $ for some constant $c$ based on the asymptotic normality of \eqref{eq:auc_emp}.
In addition, we can test \eqref{eq:null_delta_fix} by deriving the asymptotic distribution of $\psi\left (\hat{f}_{a_j}, P_0 \right) - \psi\left (\hat{f}_0; P_0 \right)$ based on the difference of the influence functions $\phi(\hat{f}_{a_j}, P_0)(X,Y) - \phi(\hat{f}_0, P_0)(X,Y)$.
To run \texttt{fsSRGP}, we can extend the above derivations to construct a fixed sequence test for testing a family of null hypotheses \eqref{eq:null_delta_fix} across multiple iterations $j$ by analyzing the  joint asymptotic distribution of the test statistics $\psi\left (\hat{f}_{a_j}, P_n \right) - \psi\left (\hat{f}_0; P_n \right)$ and compute the significance thresholds defined in \eqref{eq:sig_thres_corr}.
Similar logic can be used to derive the critical values \eqref{eq:pres_ci} and significance thresholds \eqref{eq:adapt_err} in \texttt{fsSRGP}.

\subsection{Testing model discrimination and calibration}

Section~\ref{sec:eci} considers the more complex hypothesis test \eqref{eq:eicu_hypo}, which checks for an improvement in AUC and calibration-in-the-large.
We implement this by testing three individual hypothesis tests using sequential gatekeeping.
First, we test that the difference between the average risk prediction and the observed event rate is no smaller than $-\epsilon$.
Next, we test that this difference is no larger than $\epsilon$.
Finally, we test for an improvement in AUC using the procedure described in Section~\ref{sec:auc_compare}.
To control the Type I error for rejecting the overall null hypothesis, we perform alpha spending across the individual hypotheses.
% For the empirical analyses, we allocated 0.05 of the total alpha wealth to the first hypothesis, 0.05 to the second, and 0.9 to the last.

\end{document}